\documentclass{article}

\usepackage{microtype}
\usepackage{graphicx}
\usepackage{subfigure}
\usepackage{booktabs} 

\usepackage{hyperref}

\usepackage{setspace}

\usepackage{graphicx}
\usepackage{multirow,multicol}

\usepackage[tbtags]{amsmath}
\usepackage{amsbsy}
\usepackage{amssymb}
\usepackage{amsfonts}
\usepackage{amsthm}

\usepackage{balance}
\usepackage{fancyvrb}
\usepackage{verbatim}
\usepackage{algorithm, algorithmic}

\newtheorem{lemma}{Lemma}
\newtheorem{corollary}{Corollary}

{\begin{list}               
    {$\bullet$ \hfill}{
        \setlength{\leftmargin}{\parindent}
        \setlength{\parsep}{0.04\baselineskip}
        \setlength{\itemsep}{0.5\parsep}
        \setlength{\labelwidth}{\leftmargin}
        \setlength{\labelsep}{0em}}
    }
{\end{list}}

\providecommand{\eref}[1]{\eqref{#1}}  
\providecommand{\cref}[1]{Chapter~\ref{#1}}

\providecommand{\fref}[1]{Figure~\ref{#1}}
\providecommand{\tref}[1]{Table~\ref{#1}}

\providecommand{\R}{\ensuremath{\mathbb{R}}}

\providecommand{\E}{\ensuremath{\mathbb{E}}}

\providecommand{\bydef}{\overset{\text{def}}{=}}

\renewcommand{\vec}[1]{\ensuremath{\boldsymbol{#1}}}
\providecommand{\mat}[1]{\ensuremath{\boldsymbol{#1}}}

\providecommand{\calF}{\mathcal{F}}

\providecommand{\calL}{\mathcal{L}}

\providecommand{\calN}{\mathcal{N}}

\providecommand{\calS}{\mathcal{S}}

\providecommand{\mI}{\mat{I}}

\providecommand{\vh}{\vec{h}}

\providecommand{\vx}{\vec{x}}
\providecommand{\vy}{\vec{y}}

\providecommand{\vepsilon}{\vec{\epsilon}}

\providecommand{\veta}{\vec{\eta}}

\providecommand{\fhat}{\widehat{f}}

\newcommand{\subjectto}{\mathop{\mbox{subject\, to}}}
\newcommand{\argmin}[1]{\mathop{\underset{#1}{\mbox{argmin}}}}
\newcommand{\argmax}[1]{\mathop{\underset{#1}{\mbox{argmax}}}}
\newcommand{\minimize}[1]{\mathop{\underset{#1}{\mathrm{minimize}}}}

\usepackage[not_yet_accepted]{./TeX/icml2020}

\icmltitlerunning{One Size Fits All: Can We Train One Denoiser for All Noise Levels?}

\graphicspath{{pix/}}
\begin{document}

\twocolumn[
\icmltitle{One Size Fits All: Can We Train One Denoiser for All Noise Levels?}



\icmlsetsymbol{equal}{*}

\begin{icmlauthorlist}
\icmlauthor{Abhiram Gnansambandam}{to}
\icmlauthor{Stanley H. Chan}{to,t2}
\end{icmlauthorlist}

\icmlaffiliation{to}{School of Electrical and Computer Engineering,}

\icmlaffiliation{t2}{Department of Statistics, Purdue University, West Lafayette, IN 47906, United States}

\icmlcorrespondingauthor{Stanley Chan}{stanchan@purdue.edu}

\icmlkeywords{Machine Learning, ICML}

\vskip 0.3in
]



\printAffiliationsAndNotice{This paper is presented in the 37-th International Conference on Machine Learning (ICML), Vienna, Austria, July 2020.}  

\begin{abstract}
When training an estimator such as a neural network for tasks like image denoising, it is often preferred to train \emph{one} estimator and apply it to \emph{all} noise levels. The de facto training protocol to achieve this goal is to train the estimator with noisy samples whose noise levels are uniformly distributed across the range of interest. However, why should we allocate the samples uniformly? Can we have more training samples that are less noisy, and fewer samples that are more noisy? What is the optimal distribution? How do we obtain such a distribution? The goal of this paper is to address this training sample distribution problem from a minimax risk optimization perspective. We derive a dual ascent algorithm to determine the optimal sampling distribution of which the convergence is guaranteed as long as the set of admissible estimators is closed and convex. For estimators with non-convex admissible sets such as deep neural networks, our dual formulation converges to a solution of the convex relaxation. We discuss how the algorithm can be implemented in practice. We evaluate the algorithm on linear estimators and deep networks.
\end{abstract}

\section{Introduction}
\label{Intro}

\subsection{``One Size Fits All'' Denoisers}
The following phenomenon could be familiar to those who develop learning-based image denoisers. If the denoiser is trained at a noise level $\sigma$, then its performance is maximized when the testing noise level is also $\sigma$. As soon as the testing noise level deviates from the training noise level, the performance drops  \cite{Choi_2019,Kim_access_2019}. This is a typical mismatch between training and testing, which is arguably universal for all learning-based estimators. When such a problem arises, the most straight-forward solution is to create a suite of denoisers trained at different noise levels and use the one that matches best with the input noisy image (such as those used in the ``Plug-and-Play'' priors \cite{Zhang2017_cvpr,chan2016plug,Chan_2019}). However, this ensemble approach is not effective since the model capacity is multiple times larger than necessary.

A more widely adopted solution is to train \emph{one} denoiser and use it for \emph{all} noise levels. The idea is to train the denoiser using a training dataset containing images of different noise levels. The competitiveness of these ``one size fits all'' denoisers compared to the best individually trained denoisers has been demonstrated in \cite{Zhang2017_cvpr, Zhang2017_arxiv, Mao2016, Remez2017}. However, as we will illustrate in this paper, there is no guarantee for such arbitrarily trained one-size-fits-all denoiser to have a consistent performance over the entire noise range. At some noise levels, usually at the lower tail of the noise range, the performance could be much worse than the best individuals. The cause of this phenomenon is related to how we draw the noisy samples, which is usually \emph{uniform} across the noise range. The question we ask here is that if we allocate more low-noise samples and fewer high-noise samples, will we be able to get a more consistent result?

\subsection{Objective and Contributions}
The objective of this paper is to find a sampling distribution such that for every noise level the performance is consistent. Here, by consistent we meant that the gap between the estimator and the best individuals is balanced. The idea is illustrated in \fref{fig: illustration}. The black curve in the figure represents the ensemble of the best individually trained denoisers. It is a virtual curve obtained by training the denoiser at each noise level. A typical ``one size fits all'' denoiser is trained by using noisy samples from a uniform distribution, which is denoted by the blue curve. This figure illustrates a typical in-consistence where there is a significant gap at low-noise but small gap at high noise. The objective of the paper is to find a new sampling distribution (denoted by the orange bars) such that we can achieve a consistent performance throughout the entire range. The result returned by our method is a trade-off between the overall performance and the worst cases scenarios.

\begin{figure}[!]
\centering
\includegraphics[width=\linewidth]{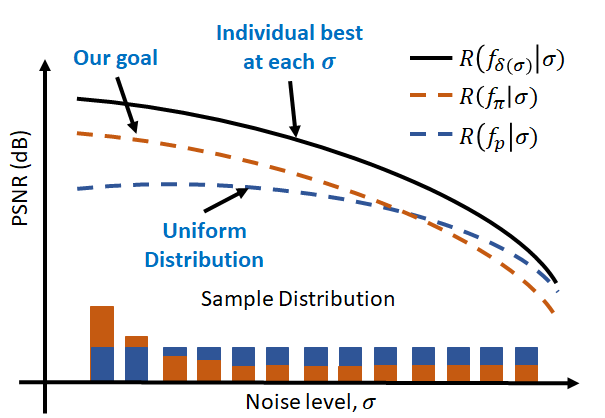}
\vspace{-4ex}
\caption{Illustration of the objective of this paper. The typical uniform sampling (blue bars) will yield a performance curve that is skewed towards one side of the noise range. The objective of this paper is to find an optimal sampling distribution (orange bars) such that the performance is consistent across the noise range. Notations will be defined in Section \ref{sec: Formulation}. We plot the risks in terms of the peak signal-to-noise ratio.}
\label{fig: illustration}
\vspace{-2ex}
\end{figure}

\vspace{4ex}
The key idea behind the proposed method is a minimax formulation. This minimax optimization minimizes the overall risk of the estimator subject to the constraint that the worst case performance is bounded. We show that under the standard convexity assumptions on the set of all admissible estimators, we can derive a provably convergent algorithm by analyzing the dual. For estimators whose admissible set is not convex, solutions returned by our dual algorithm are the convex-relaxation results. We present the algorithm, and we show that steps of the algorithm can be implemented by iteratively updating the sample distributions.

\section{Related Work}
While the above sampling distribution problem may sound familiar, its solution does not seem to be available in the computer vision and machine learning literature.

\textbf{Image Denoising}. Recent work in image denoising has been focusing on developing better neural network architectures. When encountering multiple noise levels, \cite{Zhang2017_cvpr} presented two approaches: Create a suite of denoisers at different noise levels, or train a denoiser by uniformly sampling noise levels from the range. For the former approach, \cite{Choi_2019} proposed to combine the estimators by solving a convex optimization problem. \cite{Gharbi:2016:DJD:2980179.2982399} proposed an alternative approach by introducing a noise map as an extra channel to the network. Our paper shares the same overall goal as \cite{Kim_access_2019}. However, they address problem by modifying the network structure whereas we do not change the network. Another related work is \cite{Gao_Grauman_ICCV17} which proposed an ad-hoc solution to the sample distribution. Our paper offers theoretical justification, convergence guarantee, and optimality. We should also mention \cite{Wang_Morel_2014} which scales the image intensities in order to match with the denoiser trained at a single noise level.

\textbf{Active Learning / Experimental Design}. Adjusting the distribution of the training samples during the learning procedure is broadly referred to active learning in machine learning \cite{settles2009active} or experimental design in statistics  \cite{chaloner1995bayesian}. Active learning / experimental design are typically associated with limited training data \cite{gal2017deep, sener2018active}. The goal is to optimally select the next data point (or batch of data points) so that we can estimate the model parameters, e.g., the mean and variance. The problem we encounter here is not about limited data because we can synthesize as much data as we want since we know the image formation process. The challenge is how to allocate the synthesized data.

\textbf{Constrained Optimization in Neural Network}. Training neural networks under constraints have been considered in classic optimization literature \cite{platt1988constrained} \cite{zak1995solving}. More recently, there are optimization methods for solving inequality constrained problems in neural networks \cite{pathak2015constrained}, and equality constrained problems \cite{marquez2017imposing}. However, these methods are generic approaches. The convexity of our problem allows us to develop a unique and simple algorithm.

\textbf{Fairness Aware Classification}. The task of seeking ``balanced samples'' can be considered as improving the fairness of the estimator. Literature on fairness aware classification is rapidly growing. These methods include modifying the network structure, the data distribution, and loss functions \cite{zafar2015fairness,pedreshi2008discrimination,calders2010three,hardt2016equality,kamishima2012fairness}. Putting the fairness as a constrained optimization has been proposed by \cite{zafar2017fairness}, but their problem objective and solution are different from ours.

\section{Problem Formulation}
\label{sec: Formulation}
\subsection{Training and Testing Distributions: $\pi(\sigma)$ and $p(\sigma)$}
Consider a clean signal $\vy \in \R^n$. We assume that this clean signal is corrupted by some random process to produce a corrupted signal $\vx_\sigma \in \R^n$. The parameter $\sigma$ can be treated in a broad sense as the level of uncertainty. The support of $\sigma$ is denoted by the set $\Omega$. We assume that $\sigma$ is a random variable with a probability density function $p(\sigma)$.

\textbf{Examples}. In a denoising problem, the image formation model is given by $\vx_\sigma = \vy + \sigma \veta$ where $\veta$ is a zero-mean unit-variance i.i.d. Gaussian noise vector. The noise level is measured by $\sigma$. For image deblurring, the model becomes $\vx_\sigma = \vh_{\sigma} \ast \vy + \vepsilon$ where $\vh_{\sigma}$ denotes the blur kernel with radius $\sigma$, ``$\ast$'' denotes convolution, and $\vepsilon$ is the noise. In this case, the uncertainty is associated with the blur radius. \hfill $\square$

We focus on \emph{learning-based} estimators. We define an estimator $f: \R^n \rightarrow \R^n$ as a mapping that takes a noisy input $\vx_\sigma$ and maps it to a denoised output $f(\vx_\sigma)$. We assume that $f$ is parametrized by $\theta \in \Theta$, but for notation simplicity we omit the parameter $\theta$ when the context is clear. The set of all admissible $f$'s is denoted as $\calF = \{f(\cdot, \theta) \,|\, \theta \in \Theta\}$.

To train the estimator $f$, we draw training samples from the set $\calS \bydef \{\vx_{\sigma}^{(\ell)} \;|\; \ell = 1,\ldots,N, \, \sigma \overset{\text{i.i.d.}}{\sim} \pi(\sigma)\}$, where $\ell$ refers to the $\ell$-th training sample, and $\pi(\sigma)$ is the distribution of the noise levels in the training samples. Note that $\pi$ is not necessarily the same as $p$. The distribution $\pi$ is the distribution of the \emph{training} samples, and the distribution $p$ is the distribution of the \emph{testing} samples. In most learning scenarios, we want $\pi$ to match with $p$ so that the generalization error is minimized. However, in this paper, we are purposely designing a $\pi$ which is different from $p$ because the goal is to seek an optimal trade-off. To emphasize the dependency of $f$ on $\pi$, we denote $f$ as $f_\pi$.

\subsection{Risk and Conditional Risk: $R(f)$ and $R(f|\sigma)$}
Training an estimator $f_\pi$ requires a loss function. We denote the loss between a predicted signal $f_\pi(\vx_\sigma)$ and the truth $\vy$ as $\calL(f_{\pi}(\vx_\sigma),\vy)$. An example of the loss function is the Euclidean distance:
\begin{equation}
    \calL(f_{\pi}(\vx),\vy) = \|f_{\pi}(\vx_\sigma) - \vy\|^2.
\end{equation}
Other types of loss functions can also be used as long as they are convex in $f_\pi$.

To quantify the performance of the estimator $f_\pi$, we define the notion of \emph{conditional risk}:
\begin{equation}
    R(f_{\pi} \;|\; \sigma ) \bydef \E_{(\vx_\sigma,\vy)|\sigma} \bigg[\calL(f_{\pi}(\vx_\sigma),\vy) \;|\; \sigma \bigg].
\end{equation}
The conditional risk can be interpreted as the risk of the estimator $f_\pi$ evaluated at a particular noise level $\sigma$. The overall risk is defined through iterated expectation:
\begin{align}
    R(f_{\pi})
    &\bydef
    \E_{\sigma \sim p(\sigma)} \big\{  R(f_{\pi} \;|\; \sigma ) \big\} \notag \\
    &= \int \underset{= R(f_\pi\,|\,\sigma)}{\underbrace{\E_{(\vx_\sigma,\vy)|\sigma}[\calL(f_{\pi}(\vx_\sigma),\vy) \;|\; \sigma]}} \;\; p(\sigma) d\sigma.
\end{align}
Note that the expectation of $\sigma$ is taken with respect to the true distribution $p$ since we are evaluating the estimator $f_\pi$.

\subsection{Three Estimators: $f_\pi$, $f_p$ and $f_{\delta(\sigma)}$}
The estimator $f_\pi$ is determined by minimizing the training loss. In our problem, since the training set follows a distribution $\pi(\sigma)$, it holds that $f_\pi$ is determined by
\begin{equation}
    f_\pi \bydef \argmin{f} \;\; \int R(f \,|\, \sigma) \pi(\sigma) \; d\sigma.
    \label{eq: f pi argmin}
\end{equation}
This definition can be understood by noting that $R(f \;|\; \sigma)$ is the conditional risk evaluated at $\sigma$. Since $\pi(\sigma)$ specifies the probability of obtaining a noisy samples with noise level $\sigma$, the integration in \eref{eq: f pi argmin} defines the training loss when the noisy samples are proportional to $\pi(\sigma)$. Therefore, by minimizing this training loss, we will obtain $f_\pi$.

\textbf{Example}. Suppose that we are training a denoiser over the range of $\sigma \in [a,b]$. If the training set is uniform, i.e., $\pi(\sigma) = 1/(b-a)$ for $\sigma \in [a,b]$ and is 0 otherwise, then $f_\pi$ is obtained by minimizing the sum of the individual losses $f_\pi = \text{argmin}_f \sum_{\ell=1}^N \calL(f(\vx_\sigma^{(\ell)}), \vy^{(\ell)})$ where the $N$ training samples have equally likely noise levels.  \hfill$\square$

If we replace the training distribution $\pi$ by the testing distribution $p$, then we obtain the following estimator:
\begin{equation}
    f_p = \argmin{f} \;\; \int R(f \;|\; \sigma) p(\sigma) \; d\sigma = \argmin{f} \;\; R(f).
    \label{eq: fp}
\end{equation}
Since $f_p$ minimizes the overall risk, we expect  $R(f_p) \le R(f_\pi)$ for all $\pi$. This is summarized in the lemma below.
\begin{lemma}
The risk of $f_p$ is a lower bound of the risk of all other $f_{\pi}$:
\begin{equation}
    R(f_p) \le R(f_{\pi}), \quad\quad \forall \pi.
\end{equation}
\end{lemma}
\begin{proof}
By construction, $f_p$ is the minimizer of the risk according to \eref{eq: fp}, it holds that $R(f_p) = \inf_f R(f)$. Therefore, for any $\pi$ we have $R(f_p) \le R(f_\pi)$.
\end{proof}

The consequence of Lemma 1 is that if we minimize $R(f)$ without any constraint, we will reach a trivial solution of $\pi = p$. This explains why this paper is uninteresting if the goal is to purely minimize the generalization error without considering any constraint.

Before we proceed, let us define one more distribution $\delta$ which has a point mass at a particular $\sigma$, i.e., $p(\sigma)$ is a delta function such that $p(\sigma') = \delta(\sigma'-\sigma)$. This estimator is found by simply minimizing the training loss
\begin{equation}
    f_{\delta(\sigma)} = \argmin{f} \;\; \int R(f \;|\; \sigma') \delta(\sigma' - \sigma) \; d\sigma',
    \label{eq: f delta}
\end{equation}
which is equivalent to minimizing the conditional risk $f_{\delta(\sigma)} = \argmin{f} \; R(f \,|\, \sigma)$. Because we are minimizing the conditional risk at a particular $\sigma$, $f_{\delta(\sigma)}$ gives the best individual estimate at $\sigma$. However, having the best estimate at $\sigma$ does not mean that $f_{\delta(\sigma)}$ can generalize. It is possible that $f_{\delta(\sigma)}$ performs well for one $\sigma$ but poorly for other $\sigma$'s. However, the ensemble of all these point-wise estimates $\{f_{\delta(\sigma)}\}$ will form the lower bound of the conditional risks such that $R(f_{\delta(\sigma)} \;|\; \sigma ) \le R(f_p \;|\; \sigma)$ at every $\sigma$.

\subsection{Main Problem (P1)}
We now state the main problem. The problem we want to solve is the following constrained optimization:
\begin{align}
    f^* \;\;\; \bydef&\;\; \argmin{f \in \calF} \quad\quad\;\; R(f) \tag{P1} \label{eq: main}\\
    &\subjectto     \quad\;\;\; \sup_{\sigma \in \Omega} \bigg\{ R(f|\sigma)  - R(f_{\delta(\sigma)}|\sigma) \bigg\}\le \epsilon. \notag
\end{align}
The objective function reflects our original goal of  minimizing the overall risk. However, instead of doing it without any constraint (which has a trivial solution of $f^* = f_p$), we introduce a constraint that the gap between the current estimator $f$ and the best individual $f_{\delta(\sigma)}$ is no worse than $\epsilon$, where $\epsilon$ is some threshold. The intuition here is that we are willing to sacrifice some of the overall risk by limiting the gap between $f$ and $f_{\delta(\sigma)}$ so that we have a consistent performance over the entire range of noise levels.

Referring back to \fref{fig: illustration}, we note that the black curve is $R(f_{\delta(\sigma)} | \sigma)$. The blue curve is $R(f_p|\sigma)$ for the case where $p(\sigma)$ is a uniform distribution. The orange curve is $R(f^*|\sigma)$. We show in Section \ref{sec:DualP1} that $f^*$ is equivalent to $f_\pi$ for some $\pi(\sigma)$. Note that all curves are the conditional risks.

\section{Dual Ascent}
In this section we discuss how to solve \eref{eq: main}. Solving \eref{eq: main} is challenging because minimizing over $f$ involves updating the estimator $f$ which could be nonlinear w.r.t. the loss. To address this issue, we first show that as long as the admissible set $\calF$ is convex, \eref{eq: main} is convex even if the estimators $f$ themselves are non-convex. We then derive an algorithm to solve the dual problem.

\subsection{Convexity of \eref{eq: main}}
We start by showing that under mild conditions, \eref{eq: main} is convex.

\begin{lemma}
\label{lemma: convex R}
Let $\calF$ be a closed and convex set. Then, for any convex loss function $\calL$, the risk $R(f)$ and the conditional risk $R(f|\sigma)$ are convex in $f$, for any $\sigma \in \Omega$.
\end{lemma}

\vspace{-2ex}
\begin{proof}
Let $f_1$ and $f_2$ be two estimators in $\calF$ and let $\lambda \in [0,1]$ be a constant. Then, by the convexity of $\calL$, the conditional risk $R(\cdot \,|\, \sigma)$ satisfies
\begin{align*}
R(\lambda f_1 + (1-\lambda)f_2 \,|\, \sigma) &= \E\big\{ \calL( \lambda f_1 + (1-\lambda) f_2 ) \,|\, \sigma \big\}\\
&\le \E\big\{ \lambda \calL(  f_1 ) + (1-\lambda) \calL( f_2 ) \,|\, \sigma \big\}\\
&= \lambda R(f_1|\sigma) + (1-\lambda) R(f_2|\sigma),
\end{align*}
which is convex. The overall risk $R(f)$ is found by taking the expectation of the conditional risk over $\sigma$. Since taking expectation is equivalent to integrating the conditional risk times the distribution $p(\sigma)$ (which is positive), convexity preserves and so $R(f)$ is also convex.
\end{proof}

\vspace{-2ex}
We emphasize that the convexity of $R(\cdot)$ is defined w.r.t. $f$ and not the underlying parameters (e.g., the network weights). For any convex combination of the parameters $\theta$'s, we have that $R(f(\cdot \,|\, \lambda \theta_1 + (1-\lambda)\theta_2)) \not\le \lambda R(f(\cdot \,|\, \theta_1)) + (1-\lambda)R(f(\cdot \,|\, \theta_2))$ because $f$ is not necessarily convex.

The following corollary shows that the optimization problem \eref{eq: main} is convex.
\begin{corollary}
\label{cor: convexity P1}
Let $\calF$ be a closed and convex set. Then, for any convex loss function $\calL$, \eref{eq: main} is convex in $f$.
\end{corollary}

\vspace{-2ex}
\begin{proof}
Since the objective function $R$ is convex (by Lemma \ref{lemma: convex R}), we only need to show that the constraint set is also convex. Note that the ``sup'' operation is equivalent to requiring $R(f|\sigma) - R(f_{\delta(\sigma)|\sigma}) \le \epsilon$ for \emph{all} $\sigma \in \Omega$. Since $R(f_{\delta(\sigma)}|\sigma)$ is constant w.r.t. $f$, we can define $\epsilon(\sigma)\bydef \epsilon + R(f_{\delta(\sigma)}|\sigma)$ so that the constraint becomes $R(f|\sigma) - \epsilon(\sigma) \le 0$. Consequently the constraint set is convex because the conditional risk $R(f|\sigma)$ is convex.
\end{proof}

The convexity of $\calF$ is subtle but essential for Lemma \ref{lemma: convex R} and Corollary \ref{cor: convexity P1}. In a standard optimization over $\R^n$, the convexity is granted if the admissible set is an interval in $\R^n$. In our problem, $\calF$ denotes the set of all admissible estimators, which by construction are parametrized by $\theta$. Thus, the convexity of $\calF$ requires that a convex combination of two admissible $f$'s remains admissible. All estimators based on generalized linear models satisfy this property. However, for deep neural networks it is generally unclear how the topology looks like although some recent studies are suggesting negative results \cite{Petersen_2018}. Nevertheless, even if $\calF$ is non-convex, we can solve the dual problem which is always convex. The dual solution provides the convex-relaxation of the primal problem. The duality gap is zero when the Slater's condition holds, i.e., when $\calF$ is convex and $\epsilon$ is chosen such that the constraint set is strictly feasible.

\subsection{Dual of \eref{eq: main}}\label{sec:DualP1} Let us develop the dual formulation of \eref{eq: main}. The dual problem is defined through the Lagrangian:
\begin{align}
    &L(f,\lambda) \notag\\
    &\bydef R(f) + \int \bigg\{ R(f|\sigma)  - \underset{\bydef\epsilon(\sigma)}{\underbrace{\big(R(f_{\delta(\sigma)}|\sigma) + \epsilon\big)}} \bigg\} \lambda(\sigma) d\sigma \notag\\
    &= \int R(f|\sigma) \bigg\{ p(\sigma) + \lambda(\sigma) \bigg\} d\sigma - \int \epsilon(\sigma) \lambda(\sigma) d\sigma, \label{eq: Lagrangian}
\end{align}
by which we can determine the Lagrange dual function as
\begin{align}
g(\lambda) = \inf_f L(f,\lambda) \label{eq: P1 g},
\end{align}
and the dual solution:
\begin{align}
\lambda^*
&= \argmax{\lambda \ge 0} \;  g(\lambda)  \notag \\
&= \argmax{\lambda \ge 0} \; \bigg\{\inf_f \bigg\{ \int R(f|\sigma) \big[ p(\sigma) + \lambda(\sigma) \big] d\sigma\bigg\} \notag\\
&\quad\quad\quad\quad\quad\quad\quad - \int \epsilon(\sigma)\lambda(\sigma)d\sigma\bigg\}.\label{eq: dual}
\end{align}

Given the dual solution $\lambda^*$, we can translate it back to the primal solution $\fhat$ by minimizing the inner problem in \eref{eq: dual}, which is
\begin{equation}
\fhat = \argmin{f} \; \int R(f|\sigma) \underset{\bydef \pi^*(\sigma)}{\underbrace{\bigg\{ p(\sigma) + \lambda^* (\sigma) \bigg\}}} d\sigma.
\label{eq: primal solution}
\end{equation}
This minimization is nothing but training the estimator $f$ using samples who noise levels are distributed according to $p(\sigma) + \lambda^*(\sigma)$. \footnote{For $p(\sigma) + \lambda(\sigma)$ to be a legitimate distribution, we need to normalize it by the constant $Z = \int \{p(\sigma) + \lambda(\sigma)\} d\sigma$. But as far as the minimization in  \eref{eq: primal solution} is concerned, the constant is unimportant.} Therefore, by solving the dual problem we have simultaneously obtained the distribution $\pi^*(\sigma)$, which is $\pi^*(\sigma) = p(\sigma) + \lambda^*(\sigma)$, and the estimator $\fhat$ trained using the distribution $\pi^*$.

As we have discussed, if the admissible set $\calF$ is convex then \eref{eq: main} is convex and so $\fhat$ is exactly the primal solution $f^*$. If $\calF$ is not convex, then $\fhat$ is the solution of the convex relaxation of \eref{eq: main}. The duality gap is $R(f^*) - g(\lambda^*)$.

\subsection{Dual Ascent Algorithm}
The algorithm for solving the dual is based on the fact that the point-wise $\inf_{f} \; L(f,\lambda)$ is concave in $\lambda$. As such, one can use the standard dual ascent method to find the solution. The idea is to sequentially update $\lambda$'s and $f$'s via
\begin{align}
&f^{t+1} = \argmin{f} \;\; \int R(f|\sigma) \bigg\{ p(\sigma) + \lambda^t (\sigma) \bigg\} d\sigma \label{eq: Step 1}\\
&\lambda^{t+1}(\sigma) = \bigg[\lambda^{t}(\sigma) + \alpha^t(\sigma) \bigg\{R(f^{t+1}|\sigma)  - \epsilon(\sigma)\bigg\}\bigg]_+ \label{eq: Step 2}
\end{align}
Here, $\alpha^t$ is the step size of the gradient ascent step, and $[\;\cdot\;]_+ = \max(\cdot,0)$ returns the positive part of the argument. At each iteration, \eref{eq: Step 1} is solved by training an estimator using noise samples drawn from the distribution $\pi(\sigma)^t = p(\sigma) + \lambda^t(\sigma)$. The $\lambda$-step in \eref{eq: Step 2} computes the conditional risk $R(f^{t+1}|\sigma)$ and updates $\lambda$.

Since the dual is convex, the dual ascent algorithm is guaranteed to converge to the dual solution using an appropriate step size. We refer readers to standard texts, e.g., \cite{CMU_notes}.

\section{Uniform Gap}
The solution of \eref{eq: main} depends on the tolerance $\epsilon$. This tolerance $\epsilon$ cannot be arbitrarily small, or otherwise the constraint set will become empty. The smallest $\epsilon$ which still ensures a non-empty constraint set is defined as $\epsilon_{\min}$. The goal of this section is to determine  $\epsilon_{\min}$ and discuss its implications.

\subsection{The Uniform Gap Problem (P2)}
The motivation of studying the so-called Uniform Gap problem is the inadequacy of \eref{eq: main} when the tolerance $\epsilon$ is larger than $\epsilon_{\min}$ (i.e., we tolerate more than needed). The situation can be understood from \fref{fig: illustrationP2}. For any allowable $\epsilon$, the solution returned by \eref{eq: main} can only ensure that the largest gap is no more than $\epsilon$. It is possible that the high-ends have a significantly smaller gap than the low-ends. The gap will become uniform only when $\epsilon = \epsilon_{\min}$ which is typically not known a-priori.

\begin{figure}[!]
\centering
\includegraphics[width=\linewidth]{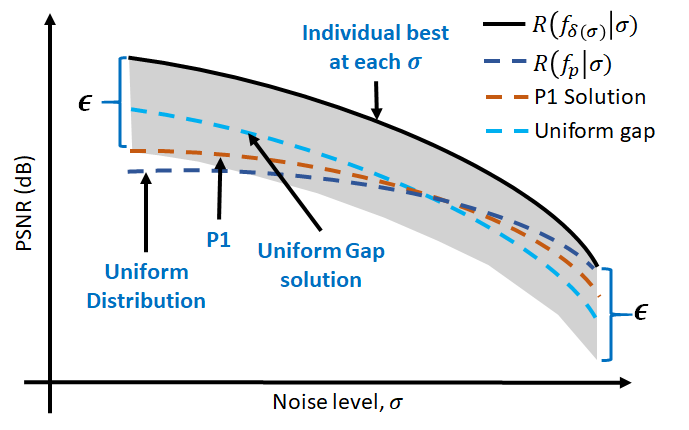}
\vspace{-4ex}
\caption{Difference between \eref{eq: main} and \eref{eq: main 2}. In \eref{eq: main}, the solution only needs to make sure that the worst case gap is upper bounded by $\epsilon$. There is no control over places where the gap is intrinsically less than $\epsilon$. Uniform Gap problem \eref{eq: main 2} addresses this issue by forcing the gap to be uniform. Note that neither \eref{eq: main} nor \eref{eq: main 2} is absolutely more superior. It is a trade-off between the noise levels, and how much we know about the testing distribution $p$.}
\label{fig: illustrationP2}
\vspace{-2ex}
\end{figure}

If we want to maintain a constant gap throughout the entire range of $\sigma$, then the optimization goal will become minimizing the maximum risk gap and not worry about the overall risk. In other words, we solve the following problem:
\begin{align}
    f^* \;\; = \argmin{f} \;\; \sup_{\sigma \in \Omega} \bigg\{ R(f|\sigma)  - R(f_{\delta(\sigma)}|\sigma) \bigg\}. \tag{P2} \label{eq: main 2}
\end{align}

When \eref{eq: main 2} is solved, the corresponding risk gap is exactly $\epsilon_{\min}$, defined as
\begin{equation}
\epsilon_{\min} \bydef \sup_{\sigma \in \Omega} \bigg\{ R(f^*|\sigma)  - R(f_{\delta(\sigma)}|\sigma) \bigg\}.
\end{equation}
The supremum in the above equation can be lifted because by construction, \eref{eq: main 2} guarantees a constant gap for all $\sigma$.

The difference between \eref{eq: main 2} and \eref{eq: main} is the switched roles of the objective function and the constraint. In \eref{eq: main}, the tolerance $\epsilon$ defines a user-controlled upper bound on the risk gap, whereas in \eref{eq: main 2} the $\epsilon$ is eliminated. Note that the omission of $\epsilon$ in \eref{eq: main 2} does not imply better or worse since \eref{eq: main} and \eref{eq: main 2} are serving two different goals. \eref{eq: main} utilizes the underlying testing distribution $p(\sigma)$ whereas \eref{eq: main 2} does not. It is possible that $p(\sigma)$ is skewed towards high noise scenarios so that a constant risk gap will suffer from insufficient performance at high-noise and over-perform at low-noise which does not matter because of $p(\sigma)$.

In practice (i.e., in the absence of any knowledge about an appropriate $\epsilon$), one can solve \eref{eq: main 2} first to obtain the tightest gap $\epsilon_{\min}$. Once $\epsilon_{\min}$ is determined, we can choose an $\epsilon > \epsilon_{\min}$ to minimize the overall risk using \eref{eq: main}.

\subsection{Algorithm for Solving \eref{eq: main 2}}
The algorithm to solve \eref{eq: main 2} is slightly different from that of \eref{eq: main} because of the omission of the constraint.

We first rewrite problem \eref{eq: main 2} as
\begin{align}
    &\minimize{f,t} \quad\quad\;\; t  \label{eq: P2 alternative} \\
    &\subjectto     \quad\;\;\; R(f|\sigma)  - \underset{\bydef r(\sigma)}{\underbrace{R(f_{\delta(\sigma)}|\sigma)}} \le t, \;\; \forall \sigma. \notag
\end{align}
Then the Lagrangian is defined as
\begin{align}
&L(f,t,\lambda)
\bydef t + \int \bigg\{ R(f|\sigma)  - r(\sigma) - t \bigg\} \lambda(\sigma) \; d\sigma \\
& = t\left(1-\int \lambda(\sigma) d\sigma\right) + \int \bigg\{R(f|\sigma) - r(\sigma)\bigg\} \lambda(\sigma) d\sigma. \notag
\end{align}
Minimizing over $f$ and $t$ yields the dual function:
\begin{align}
&g(\lambda)
\bydef \inf_{f,t} L(f,t,\lambda) \\
&=
\begin{cases}
\inf\limits_{f} \; \int \big[R(f|\sigma) - r(\sigma)\big] \lambda(\sigma) d\sigma, &\mbox{if}\; \int \lambda(\sigma) d\sigma = 1,\\
-\infty, &\mbox{otherwise}.
\end{cases}
\notag
\end{align}
Consequently, the dual problem is defined as
\begin{align}
\lambda^* &= \argmax{\lambda \ge 0} \; \inf_f \left\{\int \big[R(f|\sigma) - r(\sigma)\big] \lambda(\sigma) d\sigma\right\}
\label{eq: P2 dual}\\
&\quad\subjectto \int \lambda(\sigma) d\sigma = 1. \notag
\end{align}
Again, if $\calF$ is convex then solving the dual problem \eref{eq: P2 dual} is necessary and sufficient to determine the primal problem \eref{eq: P2 alternative} which is equivalent to \eref{eq: main 2}. The dual problem is solvable using the dual ascent algorithm, where we update $\lambda$ and $f$ according to the following sequence:
\begin{align}
f^{t+1} &= \argmin{f} \; \left\{\int \big[R(f|\sigma) - r(\sigma)\big] \lambda^t(\sigma) d\sigma\right\} \label{eq: P2 alg step 1}\\
\lambda^{t+\frac{1}{2}} &= \bigg[\lambda^t + \alpha^t  \big( R(f^{t+1}|\sigma) - r(\sigma)\big) \bigg]_+ \label{eq: P2 alg step 2}\\
\lambda^{t+1} &= \lambda^{t+\frac{1}{2}} / \int \lambda^{t+\frac{1}{2}}(\sigma)d\sigma. \label{eq: P2 alg step 3}
\end{align}
Here, \eref{eq: P2 alg step 1} solves the inner optimization in \eref{eq: P2 dual} by fixing a $\lambda$, and \eref{eq: P2 alg step 2} is a gradient ascent step for the dual variable. The normalization in \eref{eq: P2 alg step 3} ensures that the constraint of \eref{eq: P2 dual} is satisfied. The non-negativity operation $[\,\cdot\,]_+$ in \eref{eq: P2 alg step 2} can be lifted because by definition $r(\sigma) \bydef R(f_{\delta(\sigma)}|\sigma) \ge R(f|\sigma)$ for all $\sigma$. The final sampling distribution is $\pi^*(\sigma) = \lambda^*(\sigma)$.

Like \eref{eq: main}, the dual ascent algorithm for \eref{eq: main 2} has guaranteed convergence as long as the loss function $\calL$ is convex.

\section{Practical Considerations}
The actual implementation of the dual ascent algorithms for \eref{eq: main} and \eref{eq: main 2} require additional modifications. We list a few of them here.

\textbf{Finite Epochs}. In principle, the $f$-subproblems in \eref{eq: Step 1} and \eref{eq: P2 alg step 1} are determined by training a network completely using the sample distributions at the $t$-th iteration $\pi^t(\sigma) = p(\sigma) + \lambda^t(\sigma)$ and $\pi^t(\sigma) = \lambda^t(\sigma)$, respectively. However, in practice, we can reduce the training time by training the network inexactly. Depending on the specific network architecture and problem type, the number of epochs varies between 10 - 50 epochs per dual ascent iteration.

\textbf{Discretize Noise Levels}. The theoretical results presented in this paper are based on continuous distributions $\lambda(\sigma)$ and $p(\sigma)$. In practice, a continuum is not necessary since nearby noise levels are usually indistinguishable visually. As such, we discretize the noise levels in a finite number of bins. And we use the average-value of each bin as the representative noise level for the bin, so that the integration can be simplified to summation.

\textbf{$\log$-Scale Constraints}. Most image restoration applications measure the restoration quality in the log-scale, e.g., the peak signal-to-noise ratio (PSNR) which is defined as $\text{PSNR} = -10\log_{10} \text{MSE}$ where MSE is the mean squared error. Learning in the log-scale can be achieved by enforcing constraint in the log-space.

We define the the $\log$-scale risk function as:
\begin{align}
R_{\log}(f|\sigma) \bydef \E\bigg[ \log \calL( f(\vx_\sigma),  \vy) \,\big|\,\sigma \bigg]. \label{eq: Rlog}
\end{align}
With this definition, it follows the the constraints in the log-scale are represented as $\sup_{\sigma \in \Omega}\{ R_{\log}(f|\sigma)-R_{\log}(f_{\delta(\sigma)}|\sigma)\} \le \epsilon$. To turn this log-scale constraint into a linear form, we use the follow lemma by exploiting the fact that the risk gap is typically small.
\begin{lemma}
The log-scale constraint
\begin{equation}
\sup_{\sigma \in \Omega} \bigg\{ R_{\log}(f|\sigma)  - R_{\log}(f_{\delta(\sigma)}|\sigma)\bigg\} \le \epsilon
\end{equation}
can be approximated by
\begin{equation}
\sup_{\sigma\in \Omega} \bigg\{\frac{\E[\calL(f(\vx_\sigma),\vy)]}{L_{\delta}(\sigma)}\bigg\}
\le 1 + \epsilon,
\end{equation}
where $L_{\delta}(\sigma)$ is a constant (w.r.t. $f$) such that the log of $L_{\delta}(\sigma)$ equals $R_{\log}(f_{\delta(\sigma)}|\sigma)$:
\begin{equation}
\log L_{\delta}(\sigma) \bydef \E\left[ \log \calL(f_{\delta}(\vx_\sigma),\vy) | \sigma \right].
\end{equation}
\end{lemma}
\vspace{-2ex}
\begin{proof}
First, we observe that $R_{\log}(f_{\delta(\sigma)}|\sigma)$ is a deterministic quantity and is independent of $f$. Using the fact that $L_{\delta}(\sigma)$ is a deterministic constant, we can show that
\begin{align*}
&R_{\log}(f|\sigma)  - R_{\log}(f_{\delta(\sigma)}|\sigma)\\
&=\E\left[ \log \calL(f(\vx_\sigma),\vy) | \sigma \right] - \log L_{\delta}(\sigma)\\
&= \E\left[ \log \left
(\frac{\calL(f(\vx_\sigma),\vy)}{L_{\delta}(\sigma)} \right) \; \bigg| \; \sigma \right]\\
&= \E\left[ \log \left
( 1 + \frac{\calL(f(\vx_\sigma),\vy) - L_{\delta}(\sigma)}{L_{\delta}(\sigma)} \right) \; \bigg| \; \sigma \right]\\
&\approx \E\left[  \frac{\calL(f(\vx_\sigma),\vy) - L_{\delta}(\sigma)}{L_{\delta}(\sigma)} \bigg| \; \sigma\right],
\end{align*}
where we used the fact that $\calL(f(\vx_\sigma),\vy) - L_{\delta}(\sigma) \ll L_{\delta}(\sigma)$ so that $\log (1+x) \approx x$. Putting these into the constraint $R_{\log}(f|\sigma)  - R_{\log}(f_{\delta(\sigma)}|\sigma) \le \epsilon$ and rearranging the terms completes the proof.
\end{proof}

The consequence of the above analysis leads to the following approximate problem for training in the log-scale:
\begin{align}
    f^* \;\bydef& \argmin{f}\;\; R(f), \tag{P1-log}\label{eq: log problem}\\
    &\text{s.t.}     \;\; \sup_{\sigma\in \Omega} \bigg\{\frac{R(f|\sigma)} {L_{\delta}(\sigma)} \bigg\}
\le 1 + \epsilon. \notag
\end{align}
The implication of \eref{eq: log problem} is that the optimization problem with $\log$-scale constraints can be solved using the linear-scale approaches. Notice that the new distribution is now $\pi(\sigma) = p(\sigma) + \frac{\lambda(\sigma)}{L_{\delta}(\sigma)}$. The other change is that we replace $R(f_{\delta(\sigma)}|\sigma)$ with $L_\delta(\sigma)$, which are determined offline.

\section{Experiments}
We evaluate the proposed framework through two experiments. The first experiment is based on a linear estimator where analytic solutions are available to verify the dual ascent algorithm. The second experiment is based on training a real deep neural network.

\subsection{Linear Estimator}
We consider a linear (scalar) estimator so that we can access the analytic solutions. We define the clean signal as $y \sim \calN(0,\sigma_y^2)$ and the noisy signal as $x = y + \sigma \eta$, where $\eta \sim \calN(0,1)$. The estimator we choose here is $f_{\pi}(x) = a_{\pi}x$ for some parameter $a_{\pi}$ depending on the underlying sampling distribution $\pi$.

Because of the linear model formulation, we can train the estimator $\widehat{a}_{\pi}$ using closed-form equation as
\begin{equation*}
\widehat{a}_{\pi} = \argmin{a} \int \E[(ax - y)^2 | \sigma] \pi(\sigma) d\sigma = \frac{\sigma_y^2}{ \sigma_y^2 + \overline{\sigma}_{\pi}^2},
\end{equation*}
where $\overline{\sigma}_{\pi}^2 \bydef \int \sigma^2\pi(\sigma)d\sigma$. Substituting $\widehat{a}_{\pi}$ into the loss we can show that the conditional risk is
\begin{align*}
&R(f_{\pi}|\sigma)
= \E[(\widehat{a}_{\pi}x - y)^2 | \sigma] \\
&= \frac{\sigma_y^4\left[\sigma_y^2(\sigma_y^2 + \sigma^2) - 2\sigma_y^2(\sigma_y^2+\overline{\sigma}_{\pi}^2) + (\sigma_y^2+\overline{\sigma}_{\pi}^2)^2\right]}{(\sigma_y^2+\overline{\sigma}_{\pi}^2)^2}.
\end{align*}
Based on this condition risks, we can run the dual ascent algorithm to alternatingly estimate $\pi$ and $\widehat{a}_{\pi}$ according to \eref{eq: main}. \fref{fig: linear} shows the conditional risks returned by different iterations of the dual ascent algorithm. In this numerical example, we let $\sigma_y = 10$ and $\epsilon = 9$. Observe that as the dual ascent algorithm proceeds, the worst case gap is reducing \footnote{The small gap in the middle of the plot is intrinsic to this problem, since for any $\overline{\sigma}_{\pi}^2$ there always exists a $\sigma$ such that $\overline{\sigma}_{\pi}^2 = \sigma$. At this $\sigma$, the conditional risk will always touch the ideal curve.}. When the algorithm converges, it matches exactly with the theoretical solution.
\begin{figure}[!htbp]
\centering
\includegraphics[width=0.9\linewidth]{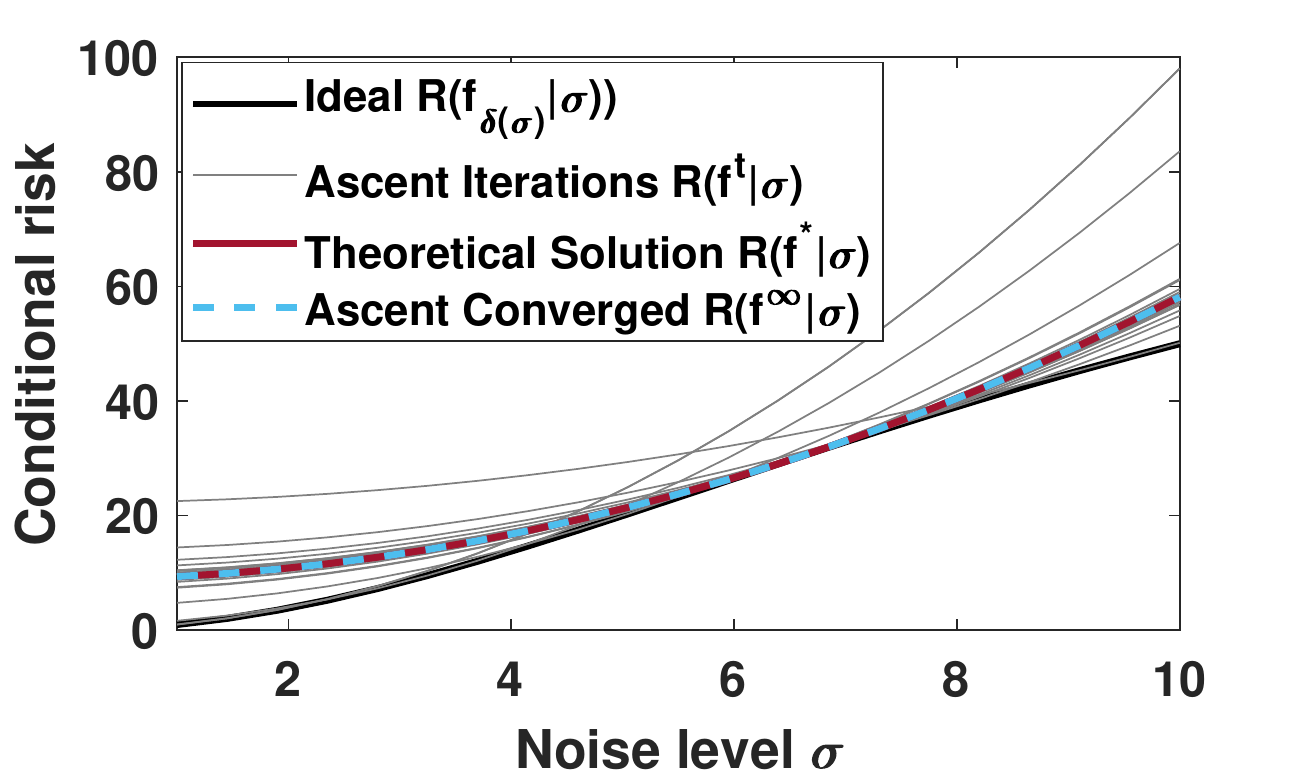}
\caption{Conditional risks of the linear problem. As the dual ascent algorithm proceeds, the risk approaches the optimal solution.}
\label{fig: linear}
\end{figure}

\begin{table*}
\centering
\footnotesize{
  \begin{tabular}[!]{ccccccccccc}
\hline
\hline
\scriptsize{\textbf{Noise level ($\sigma$)}}&0-10&10-20&20-30&30-40&40-50&50-60&60-70&70-80&80-90&90-100\\
  \hline
    &\multicolumn{10}{c}{Ideal (Best Individually Trained Denoisers)}\\
  \scriptsize{\textbf{PSNR}}&38.04&31.73&29.23&27.72&26.66&25.86&25.24&24.70&24.25&23.84\\
\hline
  &\multicolumn{10}{c}{Uniform Distribution}\\
  \scriptsize{\textbf{Distribution}}&10.0\%&10.0\%&10.0\%&10.0\%&10.0\%&10.0\%&10.0\%&10.0\%&10.0\%&10.0\%\\
  \scriptsize{\textbf{PSNR}}&37.24&31.41&29.04&27.60&26.58&25.81&25.19&24.67&24.23&23.84\\

  \hline
  &\multicolumn{10}{c}{Solution to \eref{eq: main} with 0.4dB gap}\\
\scriptsize{\textbf{Distribution}}&32.7\%&12.0\%&9.4\%&7.9\%&6.8\%&6.3\%&6.4\%&6.2\%&6.2\%&6.1\%\\
\scriptsize{\textbf{PSNR}}&37.64&31.46&29.03&27.58&26.56&25.78&25.15&24.63&24.19&23.80\\

  \hline
  &\multicolumn{10}{c}{Solution to \eref{eq: main 2}}\\
\scriptsize{\textbf{Distribution}}&81.3\%&7.6\%&3.4\%&2.0\%&1.3\%&1.0\%&0.9\%&0.9\%&0.8\%&0.8\%\\
\scriptsize{\textbf{PSNR}}&37.86&31.54&29.06&27.57&26.53&25.74&25.10&24.57&24.12&23.70\\

  \hline
  \end{tabular}
  }
\caption{\label{table:denoiser_table} Results of Section~\ref{sec:NNexp}. This table shows the PSNR values returned by one-size-fits-all DnCNN denoisers whose sample distributions are defined according to (i) uniform distribution, (ii) solution of \eref{eq: main}, and (iii) solution of \eref{eq: main 2}.}
\end{table*}

\subsection{Deep Neural Networks} \label{sec:NNexp}
The second experiment evaluates the effectiveness of the proposed framework on real deep neural networks for the task of denoising. We shall focus on the MSE loss with PSNR constraints, although our theory applies to other loss functions such as SSIM \cite{ssim} and MS-SSIM \cite{wang2003multiscale} also as long as they are convex. The noise model we assume is that $\vx_\sigma = \vy + \sigma \veta$, where $\veta \sim \mathcal{N}(0, \mI)$ with $\sigma \in [0,100]$ (w.r.t. an 8-bit signal of 256 levels). The network we consider is a 20-layer DnCNN \cite{Zhang2017_cvpr}. We choose DnCNN just for demonstration. Since our framework does not depend on a specific network architecture, the theoretical results hold regardless the choice of the networks.

The training procedure is as follows. The training set consists of 400 images from the dataset in \cite{martin2001database}. Each image has a size of $180 \times 180$. We randomly crop $50 \times 50$ patches from these images to construct the training set. The total number of patches we used is determined by the mini-batch size of the training algorithm. Specifically, for each dual ascent iteration we use 3000 mini-batches where each batch consists of 128 patches. This gives us 384k training patches per epoch. To create the noisy training samples, for each patch we add additive i.i.d. Gaussian noise where the noise level is randomly drawn from the distribution $\pi(\sigma)$. The noise generation process is done online. We run our proposed algorithm for 25 dual ascent iterations, where each iteration consists of 10 epochs. For computational efficiency, we break the noise range $[0,100]$ into 10 equally sized bins. For example, a uniform distribution corresponds to allocating 10\% of the total number of training samples per bin. The validation set consists of 12 ``standard images'' (e.g., Lena).  The testing set is the BSD68 dataset \cite{roth2005fields}, tested individually for every noise bin. The testing distribution $p(\sigma)$ for \eref{eq: main} is assumed to be uniform in \fref{fig:denoiser_plot}. Two other distributions are illustrated in \fref{fig:Uniform 25}. Notice that \eref{eq: main 2} does not require the testing distribution to be known.

The average PSNR values (conditional on $\sigma$) are reported in \tref{table:denoiser_table} and the performance gaps are illustrated in \fref{fig:denoiser_plot}. Specifically, the first two rows of the Table show the PSNR of the best individually trained denosiers and the uniform distributions. The proposed sampling distributions and the corresponding PSNR values are shown in the third row for \eref{eq: main} and the fourth row for \eref{eq: main 2}. For \eref{eq: main}, we set the tolerance level as 0.4dB. \tref{table:denoiser_table} and \fref{fig:denoiser_plot} confirm the validity of our method. A more interesting observation is the percentages of the training samples. For \eref{eq: main}, we need to allocate 32.7\% of the data to low-noise, and this percentage goes up to 81.3\% for \eref{eq: main 2}. This suggests that the optimal sampling distribution could be substantially different from the uniform distribution we use today.

\begin{figure}[!]
\centering
    \begin{tabular}{c}
    \includegraphics[width=\linewidth]{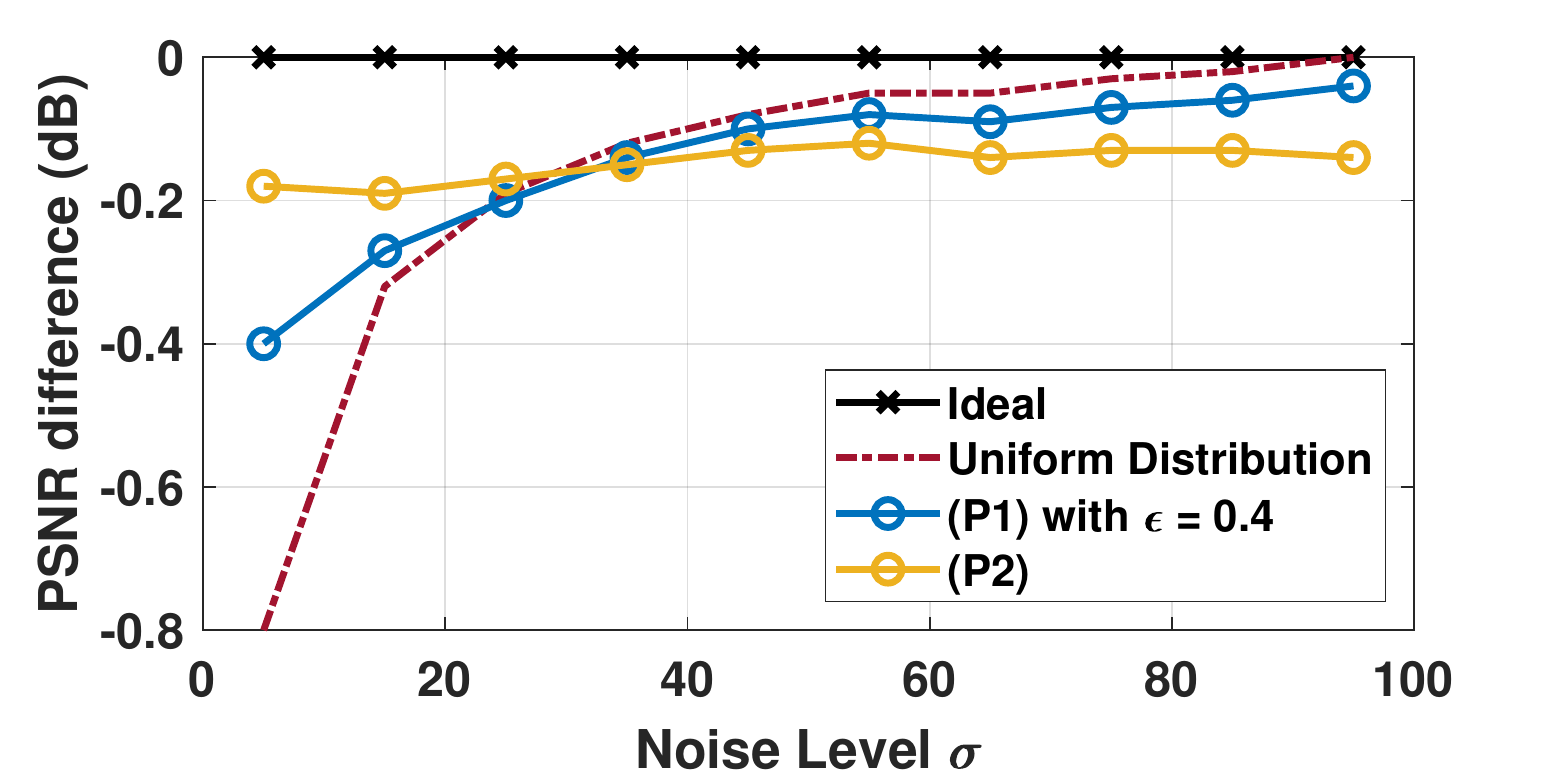}
    \end{tabular}
    \vspace{-2ex}
    \caption{\label{fig:denoiser_plot} This figure shows the PSNR difference between the one-size-fits-all denoisers and the ideal denoiser. Observe that the uniform distribution favors high-noise cases and performs poorly on low-noise cases. By using the proposed algorithm we are able to allocate training samples such that the gap is consistent across the range. \eref{eq: main} ensures that the gap will not exceed 0.4dB, whereas \eref{eq: main 2} ensures that the gap is constant.}
    \vspace{-2ex}
\end{figure}

\section{Discussions}
\subsection{Consistent gap = better?}
\begin{figure*}[t]
\centering
    \begin{tabular}{cc}
    \hspace{-2ex}\includegraphics[width=0.49\linewidth]{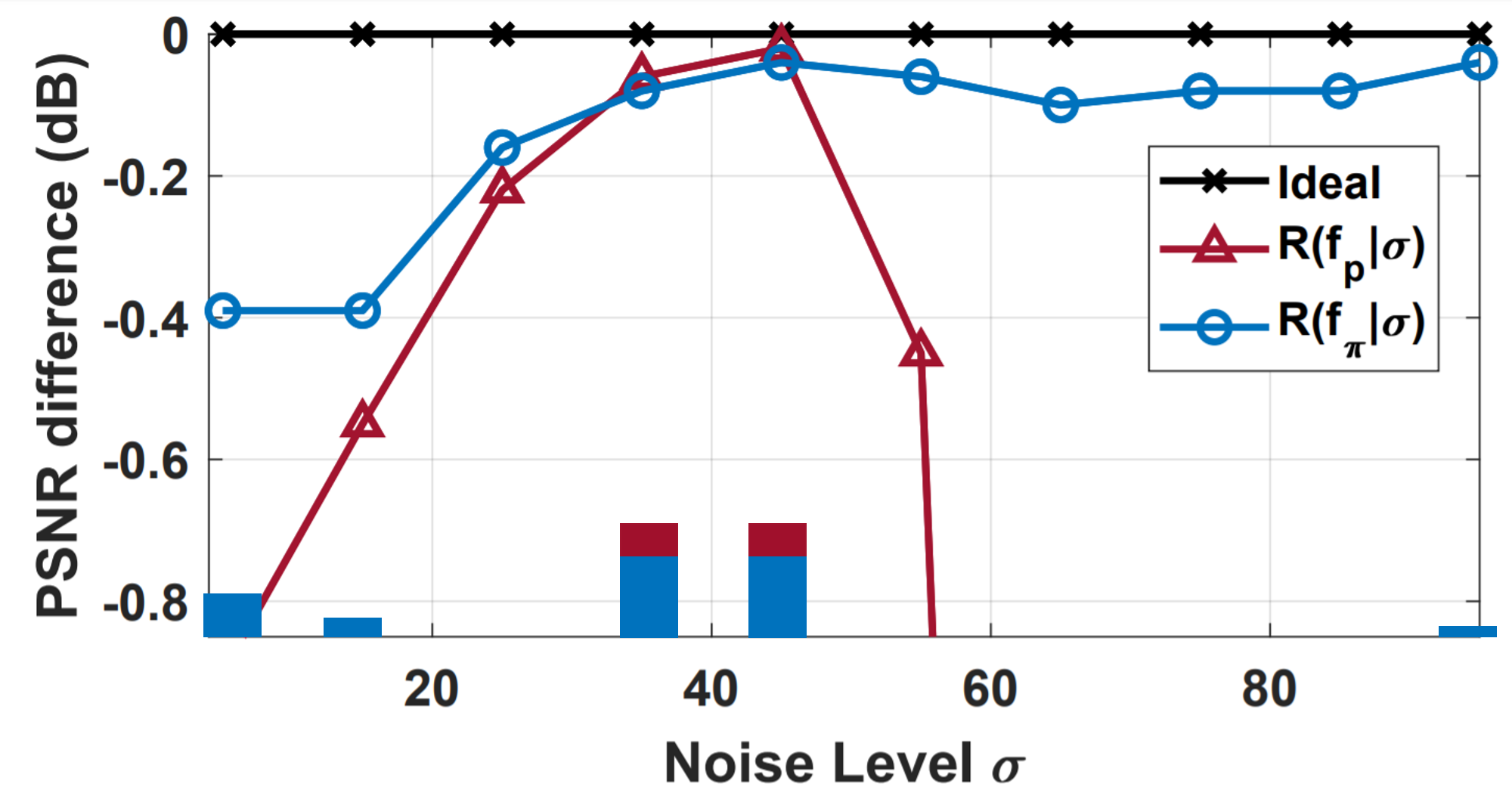}&
    \hspace{-4ex}\includegraphics[width=0.49\linewidth]{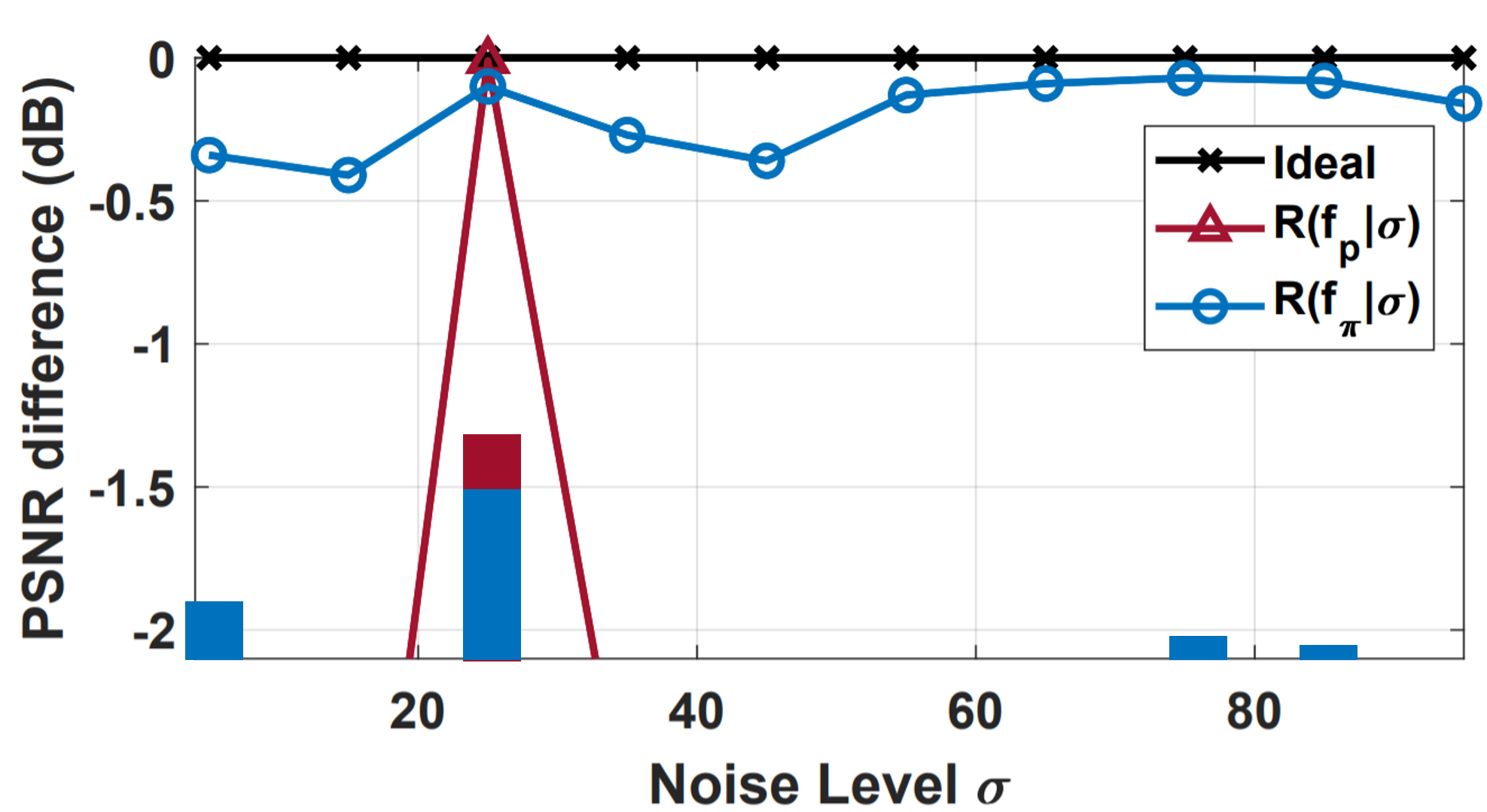}\\
     \hspace{-2ex}(a) Hypothesized distribution is Uniform[30, 50] & \hspace{-4ex}(b) Hypothesized distribution is Uniform[20, 30].
    \end{tabular}
    \vspace{-1ex}
    \caption{Usage of \eref{eq: main} when we are not certain about the true distribution. (a) The true distribution is unknown but we hypothesize that it is Uniform[30,50]. If we train the network using this distribution, we obtain the red curve. \eref{eq: main} starts with this hypothesis, and returns the blue curve. (b) Same experiment by hypothesizing that the distribution is Uniform[20,30]. Observe the robust performance of our method in both cases. The experimental setup is the same as Section~\ref{sec:NNexp}.}
\label{fig:Uniform 25}
\vspace{-2ex}
\end{figure*}

It is important to note that one-size-fits-all denosiers are about the trade-off between high-noise and low-noise cases; we offer more degrees of freedom for the low-noise cases because the high-noise cases can be learned well using fewer samples. However, achieving a consistent gap does not mean that we are doing ``better''. The solution of \eref{eq: main 2} is not necessarily ``better'' than the solution of \eref{eq: main}. The ultimate decision is application specific. If we care more about the heavy noise cases such as imaging in the dark (e.g., \cite{Gnanas_Chan_20,Choi_Elgendy_Chan_2018}) and we are willing to compromise some performance for the weak noise cases, then \eref{eq: main} could be a better than the uniform gap solution returned by \eref{eq: main 2}. Vice versa, if we know nothing about the testing distribution and we want to be conservative, then \eref{eq: main 2} is more useful.

Another consideration is how much we know about $p(\sigma)$. If we are absolutely certain that the noise is concentrated at a single value, then we should just allocate all the samples at that noise level. However, if we know something about $p(\sigma)$ but we are not absolutely sure, then \eref{eq: main} can provide the worst case performance guarantee. This is illustrated in \fref{fig:Uniform 25}, where we solved \eref{eq: main} using two hypothesized distributions. It can be observed that if we train the network using the hypothesized distributions, the performance could be bad for extreme situations. In contrast, \eref{eq: main} compromises the peak performance by offering more robust performance in other situations.

\subsection{Rule-of-thumb distribution --- the ``80-20'' rule}
Suppose that we are only looking at image denoisers with $p(\sigma)$ being uniform, and our goal is to achieve a consistent gap, then we can construct some ``rule-of-thumb'' distributions that are applicable to a few network architectures. \fref{fig:ruleot} shows three deep networks: REDNet \cite{mao2016image}, DnCNN \cite{Zhang2017_cvpr}, FFDNet \cite{Zhang2017_arxiv}, all trained using a so-called ``80-20'' rule. In this rule, we allocate the majority of the samples to the weak cases and a few to the strong cases. The exact percentage of the ``80-20'' rule is network dependent but the trend is usually similar. For example, in \fref{fig:ruleot} we allocate 70\% to [0,10], 15\% to [10,20], 8\% to [20,30], 5\% to [30,40], and 2\% to [40,50] to all the three networks. While there are some fluctuations of the PSNR differences, in general the resulting curves are quite uniform.

\begin{figure}[h]
\includegraphics[width = 1 \linewidth]{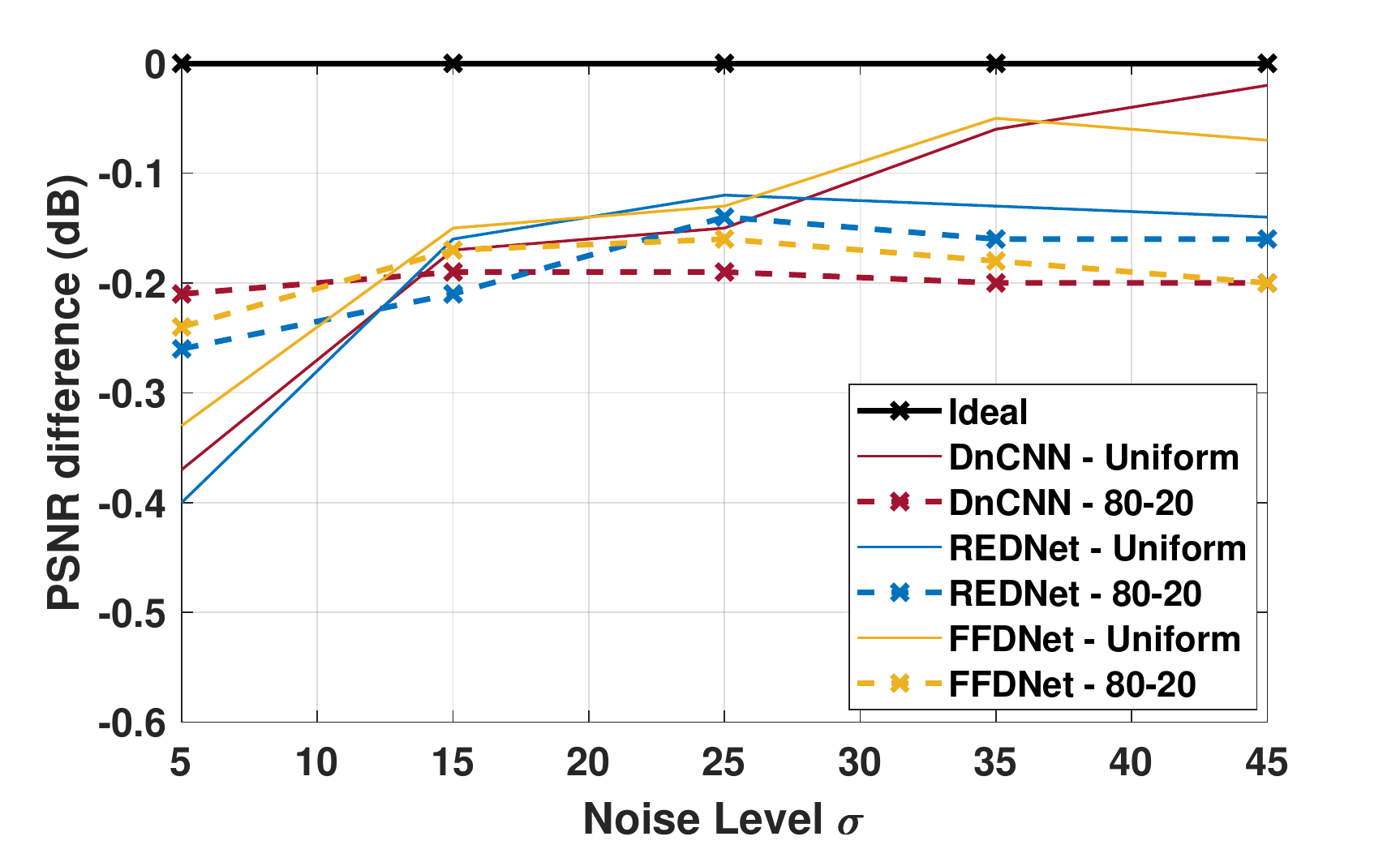}
\vspace{-4.0ex}
\caption{ \label{fig:ruleot} ``80-20'' Rule. We train the network with training samples drawn from different different noise levels according to the following distribution. $[0,10] - 70\%$, $[10,20] - 15\%$, $[20,30] - 8\%$, $[30,40] - 5\%$, $[40-50] - 2\%$. We observe that this distribution gives reasonably consistent performance at all the noise levels over all the denoisers considered here. }
\vspace{-2.0ex}
\end{figure}

\section{Conclusion}
Imbalanced sampling of the training set is arguably very common in image restoration and related tasks. This paper presents a framework which allows us to allocate training samples so that the overall performance of the one-size-fits-all denoiser is consistent across all noise levels. The convexity of the problem, the minimax formulation, and the dual ascent algorithm appear to be general for all learning-based estimators. The idea is likely to be applicable to adversarial training in classification tasks.

\section*{Acknowledgement}
The work is supported, in part, by the US National Science Foundation under grants CCF-1763896 and CCF-1718007. The authors thank Yash Sanghvi and Guanzhe Hong for many insightful discussions, and the anonymous reviewers for the constructive feedback.

\bibliographystyle{./TeX/icml2020}
\bibliography{refs}

\end{document}